\documentclass{llncs}
\usepackage{times}
\usepackage{epic}
\usepackage{ecltree}
\usepackage{graphicx}
\usepackage{latexsym}
\usepackage{amsmath}
\usepackage[noend]{algpseudocode}
\usepackage{algorithm} 

\usepackage{verbatim}
\usepackage{xspace}

\newcommand{\nina}[1]{{#1}}
\newcommand{\myomit}[1]{}

\newcommand{\constraint}[1]{\mbox{\sc #1}}
\newcommand{\regular}{\constraint{Regular}\xspace}
\newcommand{\CFG}{\constraint{Grammar}\xspace}
\newcommand{\CFGSM}{\constraint{Gr}\xspace}
\newcommand{\mymod}{\mbox{\rm mod}}
\newcommand{\NFA}{\mbox{\sc NFA}}

\newcommand{\mycase}{\constraint{Case}\xspace}

\newcommand{\among} {\mbox{\sc Among}\xspace}
\newcommand{\lex} {\mbox{\sc Lex}\xspace}
\newcommand{\CYK} {\mbox{\sc CYK}\xspace}

\newcommand{\andor}{\mbox{{\sc AND/OR}}\xspace}

\newcommand{\andnode}{\mbox{{\sc AND-node}}\xspace}
\newcommand{\ornode}{\mbox{{\sc OR-node}}\xspace}
\newcommand{\andnodes}{\mbox{{\sc AND-nodes}}\xspace}
\newcommand{\ornodes}{\mbox{{\sc OR-nodes}}\xspace}

\newcommand{\gives}{\rightarrow}

\newcommand{\calL}{\ensuremath{\mathcal{L}}}
\newcommand{\calA}{\ensuremath{\mathcal{A}}}

\begin{document}

\newlength{\halftextwidth}
\setlength{\halftextwidth}{0.47\textwidth}
\def\halffigsize{2.2in}
\def\thirdfigsize{1.5in}
\def\negvspace{0in}
\def\posvspace{0em}

\input epsf





\newcommand{\myset}[1]{\ensuremath{\mathcal #1}}

\renewcommand{\theenumii}{\alph{enumii}}
\renewcommand{\theenumiii}{\roman{enumiii}}
\newcommand{\figref}[1]{Figure \ref{#1}}
\newcommand{\tref}[1]{Table \ref{#1}}
\newcommand{\myldots}{\ldots}

\newtheorem{mydefinition}{Definition}
\newtheorem{mytheorem}{Theorem}
\newtheorem{myexample}{Example}
\newtheorem{mytheorem1}{Theorem}
\newcommand{\myproof}{\noindent {\bf Proof:\ \ }}
\newcommand{\myqed}{$\clubsuit$}
\newcommand{\simplify}{\mbox{\ensuremath{\mathit{simplify}}}}
\newcommand{\unfold}{\mbox{\ensuremath{\mathit{unfold}}}}
\newcommand{\myOmit}[1]{}
\newcommand{\tuple}[1]{\mbox{\ensuremath{\left\langle #1 \right\rangle}}}
\newcommand{\set}[1]{\mbox{\ensuremath{\left\{ #1 \right\}}}}

\title{Reformulating Global Grammar Constraints\thanks{NICTA is funded by 
the Australian Government's Department of Broadband, 
Communications,  and the Digital Economy and the 
Australian Research Council. 
}}

\author{George Katsirelos\inst{1} \and
Nina Narodytska\inst{2}
\and
Toby Walsh\inst{2}}
\institute{NICTA, Sydney, Australia,
email: george.katsirelos@nicta.com.au
\and NICTA and University of NSW,
Sydney, Australia, email: ninan@cse.unsw.edu.au, toby.walsh@nicta.com.au}

\date{10th January 2009}

\maketitle
\begin{abstract}
An attractive mechanism to specify 
global constraints in rostering
and other domains is via formal languages. For instance, 
the \regular and \CFG constraints specify
constraints in terms of the 
languages
accepted by an automaton
and a context-free grammar
respectively. 
%
Taking advantage of the fixed length
of the constraint, we give an algorithm to transform 
a context-free grammar into an automaton. 
We then 
study the use of minimization techniques
to reduce the size of such automata and speed up propagation. 
We show that minimizing such automata after
they have been unfolded and domains initially
reduced can give automata that are more compact than
minimizing before unfolding and reducing. 
Experimental results show that such transformations
can improve the size of rostering
problems that we can 
``model and run''.
\end{abstract}

\section{Introduction}

Constraint programming provides a wide range of tools for modelling
and efficiently solving real world problems. 
However, modelling
remains a challenge even for experts. 
%
%
\myOmit{A CP solver can display dramatically different 
performance on different models of the same problem. 
The modeller therefore needs considerable expertise
in specifying their problem in an appropriate model. }
Some recent attempts to simplify the modelling process have focused on
specifying constraints using formal language theory. 
For example the \regular~\cite{pesant1}
and \CFG constraints~\cite{grammar2,qwcp06}
permit constraints to be expressed
in terms of automata and grammars. 
\myOmit{A \CFG constraint can be
exponentially more succinct than the corresponding \regular constraint, but
reasoning with it is significantly more expensive. It is 
therefore not clear which is better. 
}
In this paper, we make two contributions. First, we investigate the
relationship between \regular and \CFG .
In particular, we show that it is often beneficial to reformulate
a \CFG constraint as a
\regular constraint. Second, we explore the effect of
minimizing the automaton specifying a \regular
constraint. We prove that by minimizing this automaton \emph{after}
unfolding and initial constraint propagation, we can get an
exponentially smaller and thus more efficient representation.
We show that these transformations  \myOmit{on CP and
pseudo-Boolean (PB) encodings of a challenging rostering problem,}
can improve runtimes by over an order of magnitude.

\section{Background}

A constraint satisfaction problem consists of a set of variables,
each with a domain of values, and a set of constraints
specifying allowed combinations of values for given subsets of
variables. 
A solution is an assignment 
to the variables satisfying the constraints.
A constraint is \emph{domain consistent} 
iff for each variable, every value in its domain 
can be extended to an assignment that satisfies
the constraint. 
%
We will consider
constraints specified by automata
and grammars. 
An automaton 
$A=\tuple{\Sigma, Q, q_0, F, \delta}$
consists of 
an alphabet $\Sigma$, 
a set of 
states $Q$, an initial state $q_0$, a set of 
accepting states $F$, and a transition
relation $\delta$ defining the possible next
states given a starting state
and symbol. 
The automaton is deterministic (DFA) is there is only
one possible next state, non-deterministic (NFA) otherwise.
%
%
A string $s$ is \emph{recognized} by $\calA$ iff 
starting from the state $q_0$ we 
can 
reach one of the accepting states using
the transition relation $\delta$.
%
%
Both DFAs and NFAs recognize precisely regular
languages. 
%
The constraint $\regular(\mbox{\calA},[X_1,\ldots,X_n])$
is satisfied iff $X_1$ to $X_n$ is a
string accepted by 
\mbox{\calA} \cite{pesant1}. 
Pesant has given a domain consistency propagator
for \regular based on unfolding the
DFA to give a $n$-layer automaton which
only accepts strings of length $n$ \cite{pesant1}. 

Given an automaton ${\calA}$, 
we write $\unfold_n({\calA})$
for the unfolded and layered form of ${\calA}$ that
just accepts words of length $n$ which are in the
regular language, 
$\min({\calA})$ for the
canonical form of ${\calA}$ with minimal number
of states,
$\simplify({\calA})$ for the simplified
form of ${\calA}$ constructed by deleting 
transitions and states that are no longer reachable
after domains have been reduced.
We write $f_{\calA}(n) \ll g_{\calA}(n)$ iff 
$f_{\calA}(n) \leq g_{\calA}(n)$ for all $n$,
and there exist ${\calA}$ such
that $\log \frac{g_{\calA}(n)}{f_{\calA}(n)} = \Omega(n)$.
That is, 
$g_{\calA}(n)$ is never smaller than
$f_{\calA}(n)$ and there are cases where
it is exponentially larger. 

A context-free grammar is a tuple $G = \langle T, H, P, S
\rangle$, where $T$ is a set of \emph{terminal} symbols called
the \emph{alphabet} of $G$, $H$ is a set of  \emph{non-terminal}
symbols, $P$ is a set of
productions and
$S$ is a unique starting symbol. A production is a rule $A \gives \alpha$
where $A$ is a non-terminal and $\alpha$ is a sequence of terminals and
non-terminals. A string in $\Sigma^*$ is \emph{generated} 
by $G$ if we start with the sequence $\alpha = \langle S
\rangle$ and non deterministically generate $\alpha'$ by replacing any
non-terminal $A$ in $\alpha$ by the right hand side of any production $A
\gives \alpha$ until $\alpha'$ contains only terminals.
A context free language $\mathcal{L}(G)$ 
is the language of strings 
generated by the context free grammar $G$.
A context free grammar is in Chomsky normal form if all productions
are of the form $A \gives BC$ where $B$ and $C$ are non terminals or
$A \gives a$ where $a$ is a terminal.  Any context free grammar can be
converted to one that is in Chomsky normal form with at most a linear
increase in its size. 
%
%
%
A grammar $G_a$ is \emph{acyclic} iff there exists a partial order
$\prec$ of the non-terminals, such that for every production $A_1
\rightarrow A_2 A_3$, $A_1 \prec A_2$ and $A_1 \prec A_3$.
The constraint $\CFG([X_1, \ldots, X_n],G)$
is satisfied iff $X_1$ to $X_n$ is
a string accepted by $G$~\cite{grammar2,qwcp06}.

\begin{example}
\label {e:into}
As the running example we use the $\CFG([X_1, X_2, X_3],G)$ constraint
with domains $D(X_1) =\{a\}$, $D(X_2) =\{a,b\}$,
$D(X_3) =\{b\}$ and the grammar $G$ in Chomsky
normal form~\cite{qwcp06}
$\{ S \rightarrow AB, A\rightarrow AA \mid a, B \rightarrow BB \mid b \}$.
\end{example}

%
Since we only accept strings of a fixed length, we can
convert any context free grammar to a regular grammar. However, this
may increase the size of the
grammar exponentially. Similarly, any NFA can be
converted to a DFA, but this may
increase the size of the automaton exponentially.

\section{\CFG constraint}
\label{sec:grammar}

We briefly describe the 
domain consistency
propagator for the $\CFG$ constraint proposed
in~\cite{grammar2,qwcp06}. This propagator is based on the \CYK parser
for context-free grammars.
It constructs a dynamic
programing table $V$ where an element $A$ of $V[i,j]$ is a
non-terminal that generates a substring from the domains of variables
$X_{i},\ldots,X_{i+j-1}$ that can be extended to a solution of the
constraint using the domains of the other variables. 
The 
table $V$ produced by the propagator 
for Example~\ref{e:into} is given in Figure~\ref{f:f1}.

\begin{figure} \centering
  \caption{\label{f:f1} Dynamic programming table produced by the
    propagator of the \CFG constraint. Pointers correspond to possible
    derivations.}
  \includegraphics[width=0.4\textwidth,height=3cm]{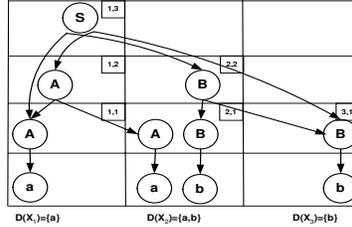}
\end{figure}

An alternative view of the dynamic programming table produced by this
propagator is as an \andor graph~\cite{qwcp07}. This is a layered DAG,
with layers alternating between \andnodes or \ornodes. Each \ornode in
the \andor graph corresponds to an entry $A \in V[i,j]$. An \ornode
has a child \andnode for each production $A \gives BC$ so that $A \in
V[i,j]$, $B \in V[i,k]$ and $C \in V[i+k,j-k]$. The children of this
\andnode are the \ornodes that correspond to the entries $B \in
V[i,k]$ and $C \in V[i+k,j-k]$. Note that the \andor graph constructed
in this manner is equivalent to the table $V$~\cite{qwcp07}, so we use
them interchangeably in this paper.

\begin{figure} \centering
  \caption{\label{f:andor-graph} \andor graph.}
  \includegraphics[width=0.4\textwidth,height=3cm]{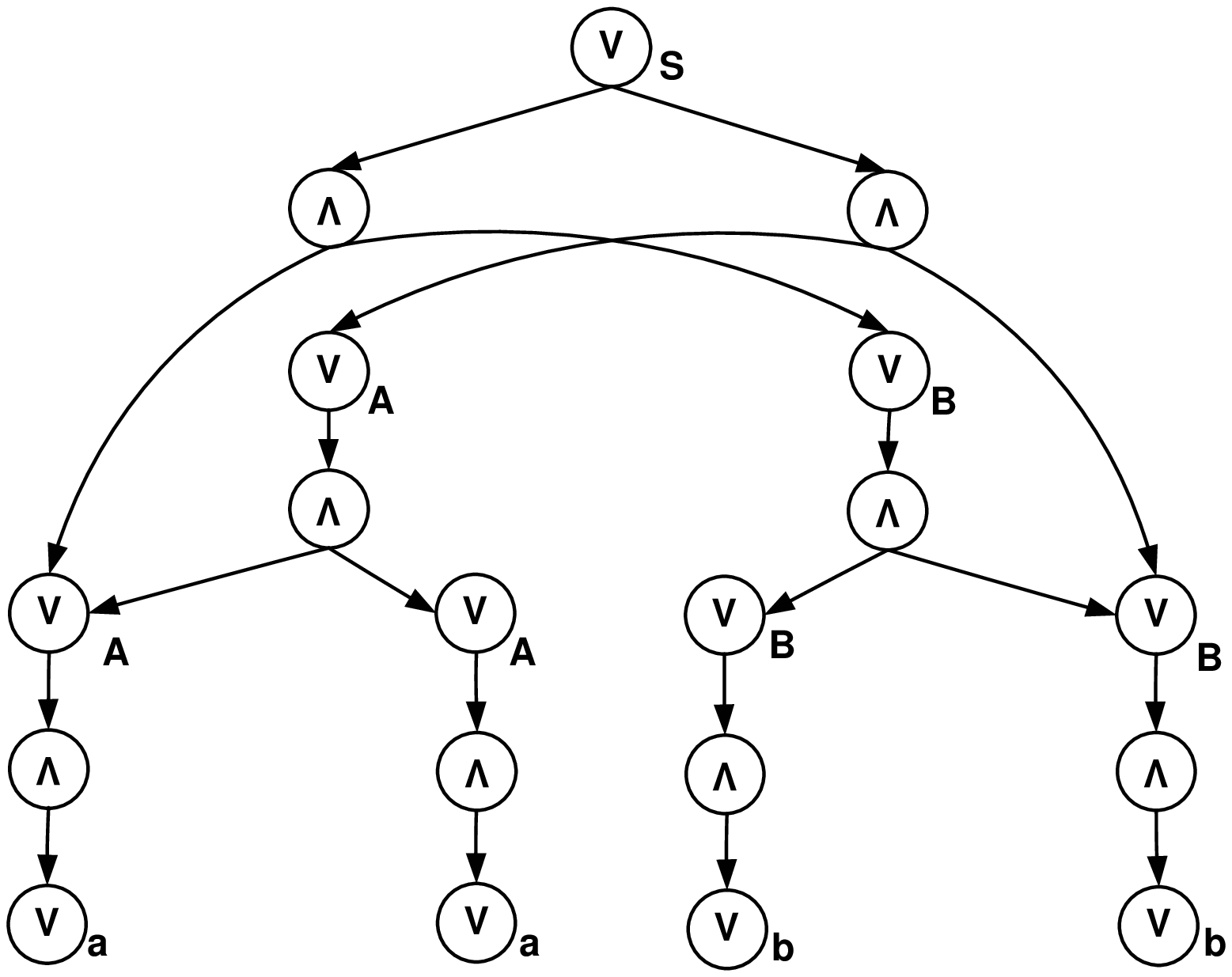}
\end{figure}

Every derivation of a string $s\in \calL(G)$ can be represented as a tree
that is a subgraph of the \andor graph and therefore can be
represented as a trace in $V$. Since every possible
derivation can be represented this way, both the table $V$
and the corresponding \andor graph are a compilation of all solutions
of the \CFG constraint.

\section {Reformulation into an automaton}
\label{sec:ref-cfg-nfa}

The time complexity
of propagating a \CFG constraint is $O(n^3|G|)$, as opposed
to $O(n|\delta|)$ for a \regular constraint. \myomit{In addition, the
monolithic propagator for the \CFG constraint is not incremental,
which means that the $O(n^3|G|)$ cost is paid at every invocation of
the propagator, while the propagator for \regular is incremental, so
$O(n|\delta|)$ is the cumulative time complexity of all invocations of
this propagator down an entire branch of the search tree.}
Therefore, reformulating
a \CFG constraint as a \regular constraint may improve propagation
speed if it does not require a large transition relation. 
In addition, we can perform optimizations such as
minimizing the automaton. 
In this section, we argue that reformulation is practical
in many cases (sections~\ref{sec:reform-cfg-acfg}-%
\ref{sec:reform-pda-nfa}), and there is a polynomial test to
determine the size of the resulting NFA 
(section~\ref{sec:computing-size-nfa}). 
In the worst case, the resulting NFA is exponentially
larger then the original \CFG constraint as the following example
shows. Therefore, performing the transformation itself is not a
suitable test of the feasibility of the approach. 


\begin{example}
\label{e:reg-cfg-sep}
Consider $\CFG([X_1,\ldots,X_n],G)$ where $G$
generates $L = \{ww^R| w \in \{0, 1\}^{n/2}\}$. 
Solutions of \CFG can be compiled into the dynamic programming table
of size $O(n^3)$, while an equivalent NFA that accepts the same 
language 
has exponential size.
Note that an exponential separation does not immediately
follow from that between regular and context-free grammars,
because solutions of the $\CFG$ constraint are the 
strict subset of $\calL(G)$ which have length $n$. 
\end{example}

In the rest of this section we describe the reformulation in three
steps.  First, we convert 
into an acyclic grammar
(section~\ref{sec:reform-cfg-acfg}), then 
into a pushdown automaton (section~\ref{s:reform-acfg-pda}), and finally
we encode 
this as a NFA
(section~\ref{sec:reform-pda-nfa}). \nina{The first two steps are well known
in formal language theory but we briefly describe them for clarity.} 
 
\subsection {Transformation 
into an acyclic grammar}
\label{sec:reform-cfg-acfg}

We first construct an acyclic grammar,
$G_a$ 
such that the language $\calL(G_a)$ coincides with solutions
of the \CFG constraint. 
Given the table $V$ produced by the $\CFG$ 
propagator (section~\ref{sec:grammar}), we construct
an acyclic grammar in the following way.  For each
possible derivation of a nonterminal $A$, $A \rightarrow BC$, such
that $A \in V[i,j]$, $B \in V[i,k]$ and $C \in V[i+k,j-k]$ we
introduce a production $A_{i,j} \rightarrow B_{i,k}C_{i+k,j-k}$ in
$G_a$
(lines~\ref{a:cfg-to-acfg:add_prods:start}-~\ref{a:cfg-to-acfg:add_prods:end}
of algorithm~\ref{a:cfg-to-acfg}). The
start symbol of $G_a$ is $S_{1,n}$.  By construction, the obtained
grammar $G_a$ is acyclic. Every production in $G_a$ is of the form
$A_{i,j} \rightarrow B_{i,k}C_{i+k,j-k}$ and nonterminals $B_{i,k}$,
$C_{i+k,j-k}$ occur in rows below $j$th row in $V$.
Example~\ref{e:acfg} shows the grammar $G_a$ obtained by
Algorithm~\ref{a:cfg-to-acfg} on our running example.

\begin{example}
  \label {e:acfg} The acyclic grammar $G_a$ 
constructed from our running example.
  \begin{align*}
    S_{1,3} \rightarrow A_{1,2}B_{3,1} \mid A_{1,1}B_{2,2}&\ & A_{1,2}\rightarrow A_{1,1}A_{2,1} &\ & B_{2,2} \rightarrow B_{2,1}B_{3,1} \\
    A_{i,1} \rightarrow a_i &\ &  B_{i,1}\rightarrow b_{i} &\ &  \forall i \in \{1,2,3\}
  \end{align*}
\end{example}

To prove  equivalence, we recall that traces of the table $V$
represent all possible derivations of $\CFG$ solutions. Therefore,
every derivation of a solution can be simulated by productions from
$G_{A}$.  For instance, consider the solution $(a,a,b)$ of \CFG from
Example~\ref{e:into}. A possible derivation of this string is
$S|_{S\in V[1,3]} \rightarrow AB|_{A \in V[1,2],B \in V[3,1]}
\rightarrow AAB|_{A \in V[1,1],A \in V[2,1],B \in V[3,1]} \rightarrow
aAB|_{\ldots} \rightarrow aaB|_{\ldots} \rightarrow aab|_{\ldots}$.
We can simulate this derivation using productions in $G_a$: $S_{1,3}
\rightarrow A_{1,2}B_{3,1} \rightarrow A_{1,1}A_{2,1}B_{3,1}
\rightarrow a_{1}A_{2,1}B_{3,1} \rightarrow a_{1}a_{2}B_{3,1}
\rightarrow a_{1}a_{2}b_{3}$.

\begin{algorithm}
\scriptsize {
\caption{Transformation 
to an Acyclic Grammar}\label{a:cfg-to-acfg}
\begin{algorithmic}[1]
\Procedure{ConstructAcyclicGrammar}{$in: X,G,V; out: G_a$}

\State $T = \emptyset$ \Comment{$T$ is the set of terminals in $G_a$} 
\State $H = \emptyset$ \Comment{$H$ is the set of nonterminals in $G_a$} 
\State $P = \emptyset$ \Comment{$P$ is the set of productions in $G_a$} 
\For{$i = 1$ \textbf{to} $n$ }	 \label{a:s_up}
	\State $V[i
	,1] = \{A |A \rightarrow a \in G, a \in D(X_i)\}$
	\For{$ A \in V[i,1]$ s.t $ A \rightarrow a \in G, a \in D(X_i)$} \label{a:first_row_cond} 	
		\State $T = T \cup \{a_{i}\}$
		\State $H = H \cup \{A_{i,1}\}$
		\State $P = P \cup \{A_{i,1} \rightarrow a_{i}\}$
	\EndFor			
\EndFor

\For{$j = 2$ \textbf{to} $n$ } \label{a:cfg-to-acfg:add_prods:start}
	\For{$i =1$ \textbf{to} $n -j + 1$}
		\For{each $A \in V[i,j]$ } 
			\For{$k = 1$ \textbf{to} $j - 1$}						
						\For{each $A \rightarrow BC \in G$ s.t. $B \in V[i,k], C \in V[i+k,j-k]$}		
						\State $H = H \cup \{A_{i,j},B_{i,k},C_{i+k,j-k}\}$
						\State $P = P \cup \{A_{i,j} \rightarrow B_{i,k}C_{i+k,j-k}\}$
						\EndFor
			\EndFor
		\EndFor
	\EndFor
\EndFor\label{a:cfg-to-acfg:add_prods:end}

\EndProcedure
\end{algorithmic}
}
\end{algorithm}

Observe that, the acyclic grammar $G_a$ is  essentially a labelling of
the \andor graph, with non-terminals corresponding to \ornodes and
productions corresponding to \andnodes. Thus, we use the notation
$G_a$ to refer to both the \andor graph and the corresponding acyclic
grammar.

\subsection {Transformation into a pushdown automaton}
\label{s:reform-acfg-pda}

Given an acyclic grammar $G_a=(T,H,P,S_{1,n})$ 
from the previous section, we now
construct a pushdown automaton $P_a (\left\langle S_{1,n} \right\rangle, 
T, T
\cup H, \delta, Q_P, F_P)$, where $\left\langle S_{1,n} \right\rangle$ 
is the initial stack
of $P_a$, $T$ is the alphabet, $T \cup H$ is the set of stack symbols,
$\delta$ is the transition function, $Q_P = F_P = \{q_P\}$ is the {single}
initial and  accepting state. We use an 
algorithm that encodes a context free grammar into a pushdown
automaton (PDA) that computes
the leftmost derivation of a string\cite{book}. 
%
The stack maintains the sequence of
symbols that are expanded in this derivation. At every step,
the PDA non-deterministically uses a production 
to expand the top symbol of the stack if it is a non-terminal,
or consumes a symbol of the input string if it matches the terminal at
the top of the stack.

We now describe this reformulation in detail. 
There exists a single state $q_P$ which is 
both the starting 
and an accepting state.
For each non-terminal $A_{i,j}$ in $G_a$ we introduce the set of
transitions $\delta(q_P, \varepsilon, A_{i,j}) = \{ (q_P, \beta)|
\forall A_{i,j} \rightarrow \beta \in G_a\}$.  For each terminal $a_i
\in G_a$, we introduce a transition $\delta(q_P, a_i, a_i) = \{(q_P,
\varepsilon)\}$.  The automaton $P_a$ accepts on the empty stack. 
This constructs a pushdown automaton
accepting $\calL(G_a)$.

\begin{example}
\label {e:apda} The pushdown automaton $P_a$ 
constructed for the running example.
  \begin{align*}
    \delta(q_P, \varepsilon,  S_{1,3})= \delta ( q_P,  A_{1,2}B_{3,1}) 
    &\ &
    \delta(q_P, \varepsilon,  S_{1,3} ) = \delta ( q_P,  A_{1,1}B_{2,2}) \\     
    \delta(q_P, \varepsilon,  A_{1,2} ) = \delta ( q_P,  A_{1,1}A_{2,1}) &\ & 
    \delta(q_P, \varepsilon,  B_{2,2} ) = \delta ( q_P,  B_{2,1}B_{3,1})\\
    \delta(q_P, \varepsilon,  A_{i,1} ) = \delta ( q_P,  a_i) &\ &  \delta(q_P, \varepsilon,  B_{i,1} ) = \delta ( q_P,  b_i) \forall i \in \{1,2,3\}\\
    \delta(q_P, a_i,  a_i) = \delta ( q_P,  \varepsilon) &\ &  \delta(q_P, b_i,  b_{i} ) = \delta ( q_P,  \varepsilon) \forall i \in \{1,2,3\}
  \end{align*}
\end{example}

\subsection {Transformation into a NFA}
\label{sec:reform-pda-nfa}

Finally, we construct an $\NFA(\Sigma, Q, Q_0 ,F_0,\sigma)$, denoted $N_a$,
using the PDA from the last section.   States of this
NFA encode all possible configurations of the stack of the PDA
that can appear in parsing a string from $G_a$. To
reflect that a state of the NFA represents a stack, we write states as
sequences of symbols $\left\langle \alpha \right\rangle$, where
$\alpha$ is a possibly empty sequence of symbols and 
$\alpha[0]$ is the top of the
stack. For example, the initial state 
is $\left\langle S_{1,n} \right\rangle$ corresponding to the initial
stack $\left\langle S_{1,n} \right\rangle$ of $P_a$.  
Algorithm~\ref{a:pda-to-nfa}
unfolds the PDA in a similar way to 
unfolding the DFA. \nina{Note that the NFA accepts
only strings of length $n$ and has the initial state $Q_0 =\left\langle S_{1,n} \right\rangle$
and the single final state $F_0 = \left\langle \right\rangle$.}

\begin{algorithm}
\scriptsize {
\caption{Transformation to NFA}\label{a:pda-to-nfa}
\begin{algorithmic}[1]
\Procedure{PDA to NFA}{$in: P_a, out: N_a$}
   \State $Q_{u} = \{\left\langle S_{1,n} \right\rangle\}$ \Comment{$Q_u$ is the set of unprocessed states}
   \State $Q = \emptyset$ \Comment{$Q$ is the set of states in $N_a$}
   \State $\sigma = \emptyset$ \Comment{$\sigma$ is the set of transitions in $N_a$}
   \State $Q_0 = \left\{ \left\langle S_{1,n}\right\rangle \right\}$ \Comment{$Q_0$ is the initial state in $N_a$}
   \State $F_0 = \left\{ \left\langle \right\rangle \right\}$ \Comment{$F_0$ is the set of final states in $N_a$}
   \While{$Q_u$ is not empty}
      \If{$q \equiv \left\langle A_{i,j}, \alpha \right\rangle$}
          \For{each transition $\delta (q_P, \varepsilon , A_{i,j}) = (q_P, \beta) \in \delta$}
              \State $\sigma= \sigma\cup \{\sigma(\left\langle A_{i,j}, \alpha \right\rangle, \varepsilon) = \left\langle \beta, \alpha \right\rangle\}$
              \If {$\left\langle \beta, \alpha \right\rangle \notin Q$}
              	\State $Q_u = Q_u \cup \{\left\langle \beta, \alpha \right\rangle\}$
              \EndIf
          \EndFor
           \State $Q = Q \cup \{\left\langle A_{i,j}, \alpha \right\rangle\}$

      \ElsIf{$q \equiv \left\langle a_i, \alpha \right\rangle$}
           \For{each transition $\delta(q_P, a_i, a_i) = (q_P, \varepsilon)  \in \delta$}
                \State $\sigma= \sigma\cup \{\sigma(\left\langle a_{i}, \alpha \right\rangle, a_{i}) = \left\langle \alpha \right\rangle\}$
              	\If {$\left\langle \alpha \right\rangle \notin Q$}
                	\State $Q_u = Q_u \cup \{\left\langle \alpha \right\rangle\}$
              	\EndIf
           \EndFor
          \State $Q = Q \cup \{\left\langle  a_i, \alpha \right\rangle\}$
       \EndIf
   \State $Q_u = Q_u \setminus \{q\}$
   \EndWhile 
   \State $N_a(\Sigma,Q, Q_0, F_0, \sigma) = \varepsilon-Closure(N_a (\Sigma,Q,Q_0, F_0\sigma))$. \label{a:cfg2nfa_eps_closure} 
\EndProcedure
\end{algorithmic}
}
\end{algorithm}

We start from the initial stack $\tuple{S_{1,n}}$
and find all distinct stack configurations that are reachable from
this stack using transitions from $P_a$. For each reachable 
stack configuration we create a state in the NFA 
and add the corresponding transitions.
If the new stack configurations are the result of expansion
of a production in the original grammar, these transitions 
are $\varepsilon-$transitions, otherwise they consume a symbol
from the input string.
Note that if a non-terminal appears on
top of the stack and gets replaced, then it cannot appear in 
any future stack configuration
due to the acyclicity of $G_a$. Therefore $|\alpha|$ is bounded
by $O(n)$ and Algorithm~\ref{a:pda-to-nfa} terminates. The size of
$N_a$ is $O(|G_a|^n)$ in the worst case.  The automaton $N_a$ that we
obtain before line~\ref{a:cfg2nfa_eps_closure} is an acyclic
NFA with $\varepsilon$ transitions. It accepts
the same language as the PDA $P_a$ since every path between the
starting and the final state of $N_A$ is a trace of the stack
configurations of
$P_a$. Figure~\ref{f:f2}(a) shows the automaton $N_a$ with
$\varepsilon$-transitions constructed from the running example.
%
%
%
After applying the $\varepsilon$-closure operation, we obtain a
layered NFA that does not have $\varepsilon$
transitions (line~\ref{a:cfg2nfa_eps_closure}) (Figure~\ref{f:f2}(b)).

\begin{figure} \centering
    \caption{\label{f:f2}
   $N_a$ produced by Algorithm~\ref{a:pda-to-nfa}}
    \includegraphics[width=1\textwidth]{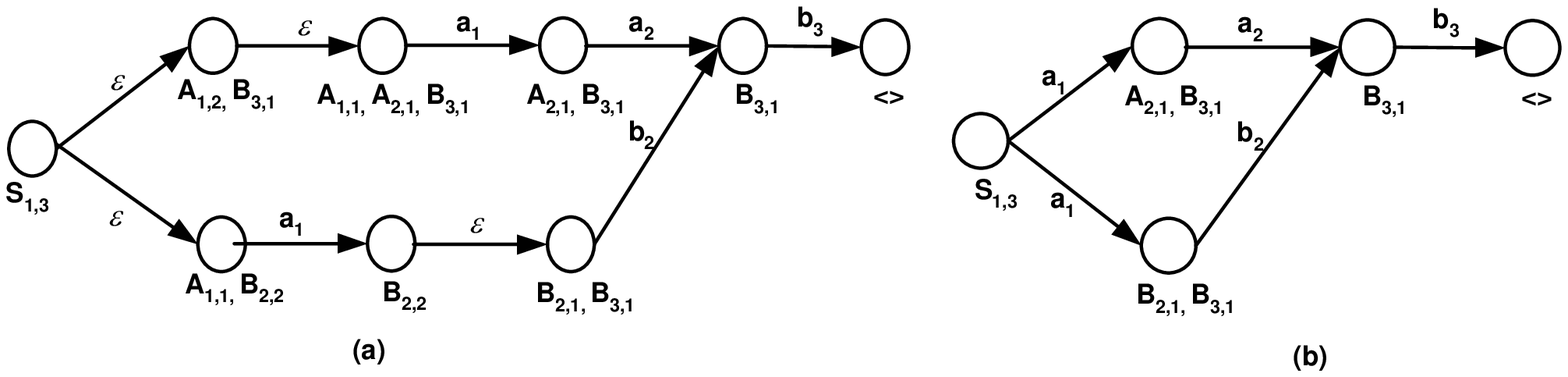}
\end{figure}

\subsection{Computing the size of the NFA} 
\label{sec:computing-size-nfa}

As the NFA may be exponential in size, we provide
a polynomial method of computing its size in advance. 
We can use this to decide if it is practical to transform it 
in this way. 
Observe first that the transformation of a PDA to an NFA maintains
a queue of states that correspond to stack configurations. Each state
corresponds to an \ornode in the \andor graph and each state of an
\ornode $v$ is generated from the states of the parent \ornodes of
$v$. This suggests a relationship between paths in the \andor graph 
of the \CYK algorithm and
states in $N_a$. 
We use  this relationship
to compute a loose upper bound for the number of states in $N_a$
in time linear in the size of the \andor graph by counting the number
of paths in that graph.
Alternatively, we compute
the exact number of states in $N_a$
in time quadratic in the size of the \andor graph. 

\begin{theorem}
  \label{thm:path-to-stack}
  There exists a surjection between paths in $G_a$ from the root to
  \ornodes and stack configurations in the PDA $P_a$.
\end{theorem}

\begin{proof}
Consider a path $p$ from the root of the \andor graph to an \ornode
labelled with $A_{i,j}$. We construct 
a stack configuration $\Gamma(p)$ 
that corresponds to $p$. 
We start with the empty stack $\Gamma = \left\langle
\right\rangle$. We traverse the path from the root to $A_{i,j}$. For
every \andnode $v_1 \in p$, with left child $v_l$ and right child
$v_r$, if the successor of $v_1$ in $p$ is $v_l$, then we push $v_r$
on $\Gamma$, otherwise do nothing. When we reach $A_{i,j}$, we push it on
$\Gamma$. The final configuration $\Gamma$ is unique for $p$ and
corresponds to the stack of the PDA after having parsed the substring
$1\ldots i-1$ and having non-deterministically chosen to parse the
substring $i \ldots i+j-1$ using a production with $A_{i,j}$ on the
LHS. 

\sloppy
We now show that all stack configurations can be generated by the
procedure above. \nina{
Every stack configuration corresponds to at least one partial left most derivation of a string. We say a stack configuration
$\left\langle \alpha \right\rangle$ corresponds to a derivation $dv = \left\langle a_1,\ldots,a_{k-1},A_{k,j}, \alpha \right\rangle$  
if  $\alpha$ is the context of the stack after parsing the prefix of the string of length $k+j$. Therefore, it is enough to show that all partial left most derivation (we omit the prefix of terminals) 
can be generated by the procedure above. We prove by a contradiction.
Suppose that $\left\langle a_1,\ldots,a_{i-1}, B_{i,j}, \beta \right\rangle$ is the partial left most
derivation  such that $\Gamma (p(root, B_{i,j})) \neq \beta$,  where
$p (root, B_{i,j})$  is the path from the root to the \ornode $B_{i,j}$ and 
for any partial derivation $\left\langle a_1,\ldots,a_{k-1}, A_{k,j}, \alpha \right\rangle$, 
such that $k < i$ $A_{k,j} \in G_a$  $\Gamma (p(root, A_{k,j})) = \alpha$.
%
Consider the production rule that introduces the nonterminal $B_{i,j}$ to the partial derivation.
If the production rule is $D \rightarrow  C, B_{i,j}$, then the partial derivation is
$\left\langle a_1,\ldots,a_{f}, D, \beta \right\rangle \Rightarrow \mid_{ D \rightarrow  C, B_{i,j}}  \left\langle a_1,\ldots,a_{f}, C, B_{i,j}, \beta \right\rangle$. 
The path from the root to the node $B_{i,j}$ is a concatenation of the paths from $D$ to $B_{i,j}$ and from the root
to $D$. Therefore, $\Gamma(p(root, B_{i,j}))$  is constructed as a concatenation of $\Gamma(p(D, B_{i,j}))$ and $\Gamma(p(root,D))$. $\Gamma(p(D, B_{i,j}))$ is empty because the node $B_{i,j}$ is  the right child of 
$\andnode$ that corresponds to the production $D \rightarrow  C, B_{i,j}$ and  $\Gamma(p(root, D)) = \beta$ because $f < i$. 
Therefore, $\Gamma(p(root, B_{i,j}))= \beta$.   
If the production rule is $D \rightarrow  B_{i,j}, C$, then the partial derivation is
$ \left\langle a_1,\ldots,a_{i-1}, D, \gamma \right\rangle \Rightarrow \mid_{ D \rightarrow  B_{i,j}, C}  \left\langle a_1,\ldots,a_{i-1}, B_{i,j}, C, \gamma\right\rangle =  \left\langle a_1,\ldots,a_{i-1}, B_{i,j}, \beta \right\rangle $.
Then, $\Gamma(p(root, D)) = \gamma$, because $i- 1 < i$ and $ \Gamma(p(D, B_{i,j})) = \left\langle C \right\rangle$, because the node $B_{i,j}$ is  the left child of $\andnode$ that corresponds to the production $D \rightarrow  C, B_{i,j}$.
Therefore, $\Gamma(p (root, B_{i,j})) =  \left\langle C, \gamma \right\rangle =  \beta$. This leads to a contradiction.}
 
\myomit{
Recall that each state in $Q_u$ of
algorithm~\ref{a:pda-to-nfa} corresponds to an \ornode $A_{i,j}$ of
the \andor graph and a label $\alpha$. For each such state, there
exists a stack configuration $\Gamma=\left\langle A_{i,j} \; \alpha
\right\rangle$. Moreover, each label is generated from the labels of
the parent \ornodes of $A_{i,j}$. Therefore, each label can be mapped
to at least one path to $A_{i,j}$ and each stack configuration
$\Gamma$ can also be mapped to at least one path in $G_a$.}
\qed
\end{proof}

\begin{example}
  An example of the mapping described in the last proof
  is in Figure~\ref{f:f3}(a) for the
  grammar of our running example. Consider the \ornode
  $A_{1,1}$. There are 2 paths from $S_{1,3}$ to $A_{1,1}$. One is
  direct and uses only \ornodes $\tuple{S_{1,3}, A_{1,1}}$ and the
  other uses \ornodes $\tuple{S_{1,3}, A_{1,2}, A_{1,1}}$. The
  2 paths are mapped to 2 different stack configurations
  $\tuple{A_{1,1},B_{2,2}}$ and $\tuple{A_{1,1}, A_{2,1}, B_{3,1}}$
  respectively. We highlight edges that are incident to \andnodes on
  each path and lead to the right children of these \andnodes. There
  is exactly one such edge for each element of a stack
  configuration. \qed
\end{example}

Note that theorem~\ref{thm:path-to-stack} only specifies a surjection
from paths to stack configurations, not a bijection. Indeed, different
paths may produce the same configuration $\Gamma$. 

\begin{example}
  Consider the grammar $G= \{ S\gives AA, A \gives a | AA | BC, B\gives
  b|BB, C \gives c | CC \}$ and
  the \andor graph of this grammar for a string of length
  5. The path $\tuple{S_{1,5}, A_{2,4}, B_{2,2}}$ uses the productions
  $S_{1,5} \gives  A_{1,1} A_{2,4}$ and $A_{2,4} \gives B_{2,2}
  C_{4,2}$, while the path $\tuple{S_{1,5}, A_{3,3}, B_{3,1}}$ uses the
  productions $S_{1,5} \gives A_{1,2} A_{3,3}$ and $A_{3,3} \gives
  B_{3,1} C_{4,2}$. Both paths map to the same stack configuration
  $\tuple{C_{4,2}}$. 
  \qed
\end{example}

By 
construction, 
the resulting
NFA 
has one state for each stack configuration of the PDA in
parsing a string. 
Since each path corresponds to 
a stack configuration, 
the number of states of the NFA 
before applying $\varepsilon$-closure
is
bounded 
by the number of paths from the root to any \ornode in the \andor
graph. 
This is cheap to compute using the following recursive
algorithm~\cite{darwiche01}:

\begin{eqnarray}
  \label{eq:path-counting}
  PD(v) = \left\{
    \begin{array}{lr}
      1 & \textrm{If $v$ has no incoming edges} \\
      \sum_p PD(p) & \textrm{ where $p$ is a parent of $v$} \\
    \end{array}
    \right.
\end{eqnarray}

Therefore, the number of states of the NFA $N_a$ is at most $\sum_{v}
PD(v)$, where $v$ is an \ornode of $G_a$ (Figure~\ref{f:f3}).

\begin{figure} \centering
    \caption{\label{f:f3}
   Computing the size of $N_a$. (a) \andor graph $G_a$. (b) Stack graph $G_{A_{1,1}}$
   }
    \includegraphics[width=0.5\textwidth,height=4cm]{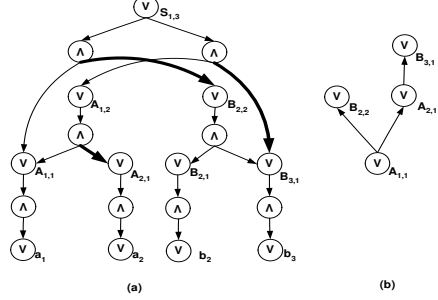}
\end{figure}

We can compute the exact number of paths in $N_a$ before 
$\varepsilon$-closure without constructing the NFA by 
counting paths in the \emph{stack graph} $G_v$ for each \ornode
$v$. 
The stack graph captures the observation that each element of a
stack configuration generated from a path $p$ is associated with
exactly one edge $e$ that is incident on $p$ and leads to the right
child of an \andnode. $G_v$ contains one path for each sequence of
such edges, so that if two paths $p$ and $p'$ in $G_a$ are mapped to
the same stack configuration, they are also mapped to the same path in
$G_v$. 
\nina{%
Formally, the stack graph of an \ornode $v \in V(G_a)$ is a DAG $G_v$, 
such that
for every stack configuration $\Gamma$ of $P_a$ with $k$ elements, 
there is exactly one path $p$ in $G_v$ of length $k$
and $v'$ is the $i^{th}$ vertex of $p$ if and only if $v'$ 
is the $i^{th}$ element from the top of $\Gamma$.
}

\begin{example}
  Consider the grammar of the running example and the \ornode
  $A_{1,1}$ in the \andor graph. The stack graph $G_{A_{1,1}}$ for
  this \ornode is shown in figure~\ref{f:f3}(b). Along the
  path $\tuple{S_{1,3} A_{1,1}}$, only the edge that leads to
  $B_{2,2}$ generates a stack element. This edge is mapped to the edge
  $(A_{1,1}, B_{2,2})$ in $G_{A_{1,1}}$. Similarly, the edges that
  lead to $A_{2,1}$ and $B_{3,1}$ are mapped to the edges $(A_{1,1},
  A_{2,1})$ and $(A_{2,1}, B_{3,1})$ respectively.
  \qed
\end{example}

Since
$G_v$ is a DAG, we can efficiently count the number of paths
in it.
We construct $G_v$ using algorithm~\ref{a:stack-graph}.
The graph $G_v$ computed in algorithm~\ref{a:stack-graph} for an
\ornode $v$ has as many paths as there are unique stack
configurations in
$P_a$ with $v$ at the top. 
\begin{algorithm}
\scriptsize{ 
\caption{Computing the stack DAG $G_v$ of an \ornode $v$}\label{a:stack-graph}
\begin{algorithmic}[1]
\Procedure{StackGraph}{$(in: G_a, v, out: G_v)$}
   \State $V(G_v) = \{ v \}$
   \State $label(v) = \{ v \}$
   \State $Q = \left\{(v, v_p) | v_p \in \mathit{parents}(v)\right\}$ 
   \Comment{queue of edges}
   \While{$Q$ not empty}
       \State $(v_c, v_p) = \mathit{pop}(Q)$
       \If{$v_p$ is an \andnode $v_c$ is left child of $v_p$} \label{a:i-g:left-child}
           \State $v_r = \mathit{children_r(v_p)}$
           \State $V(G_v) = V(G_v) \cup \{ v_r \} $
           \State $E(G_v) = E(G_v) \cup \left\{ (v_l, v_r) | v_l \in label(v_c) 
                                        \right\}$
           \State $label(v_p) = label(v_p) \cup \{ v_r \}$
       \Else
           \State $label(v_p) = label(v_p) \cup label(v_c)$
       \EndIf
       \State
       $Q = Q \cup \left\{(v_p, v_{p}') | v_p' \in \mathit{parents}(v_p)\right\}$
       \Comment{}
   \EndWhile
\EndProcedure
\end{algorithmic}
}
\end{algorithm}

\begin{theorem}
  There exists a bijection between paths in $G_v$ and states in the
  NFA $N_a$ which correspond to stacks with $v$ at the top.
\end{theorem}

\begin{proof}

  \nina{%
  Let $p$ be a path from the root to $v$ in $G_a$.
  First, we show that every path $p'$ in $G_v$
  corresponds to a stack configuration, by mapping $p$ to $p'$. 
  Therefore $p'$ corresponds to $\Gamma(p)$.
  We then show that $p'$ is unique for $\Gamma(p)$.
  This establishes a bijection between paths in $G_v$
  and stack configurations.
  }

  \nina{%
  We traverse the inverse of $p$, denoted $inv(p)$ and construct $p'$
  incrementally. Note that every vertex in $inv(p)$ is examined by
  algorithm~\ref{a:stack-graph} in the construction of $G_v$. If
  $inv(p)$ visits the left child of an \andnode, we append the right
  child of that \andnode to $p'$. This vertex is in $G_v$ by
  line~\ref{a:i-g:left-child}. 
  By the construction of $\Gamma(p)$ in the proof of
  theorem~\ref{thm:path-to-stack}, a symbol is placed on the stack if
  and only if it is the right child of an \andnode, hence if and only
  if it appears in $p'$.
  Moreover, if a vertex is the $i^{th}$ vertex in a path, it
  corresponds to the $i^{th}$ element from the top of $\Gamma(p)$.
  We now see that $p'$ is unique for $\Gamma(p)$. 
  Two distinct paths of length $k$ cannot map to the same
  stack configuration, because they must differ in at least one
  position $i$, therefore they correspond to stacks with different
  symbols at position $i$.
}
  Therefore, there exists a bijection between paths in
  $G_v$ and stack configurations with $v$ at the top.
  \qed
\end{proof}

Hence $|Q(Ng)| = \sum_v \#\mathit{paths}(G_v)$, where $v$ is an
\ornode of $G_a$.
Computing the stack graph $G_v$ of
every \ornode $v$ takes $O(|G_a|)$ time, as does counting paths in
$G_v$. Therefore, computing the number of
states in $N_a$ takes $O(|G_a|^2)$ time.
We can also compute the number of states in the $\varepsilon$-closure
of $N_a$ 
by observing that if none of the \ornodes that are reachable by paths
of length 2 from an \ornode $v$ correspond to terminals, then any
state that corresponds to a stack configuration with $v$ at the top
will only have outgoing $\varepsilon-$transitions and will be removed
by the $\varepsilon-$closure. Thus, to compute the number of states in
$N_a$ after $\varepsilon-$closure, we sum the number of paths in $G_v$
for all \ornodes $v$ such that a terminal \ornode can be reached from
$v$ by a path of length 2.



\subsection{Transformation into a DFA}

Finally, we
convert the NFA into a DFA using the 
standard subset construction. 
This is optional as Pesant's propagator
for the \regular constraints works just as 
well with NFAs as DFAs. Indeed, removing
non-determinism may increase the size of the
automaton and slow down propagation. 
However, converting into a DFA opens up the possibility
of further optimizations. In particular, as we 
describe in the next section, 
there are efficient methods to minimize the
size of a DFA. By comparison,
minimization of a NFA is PSPACE-hard in general
\cite{msfocs72}. Even when we consider just the
acyclic NFA constructed by unfolding
a NFA, minimization remains 
NP-hard \cite{ajvwia99}. 

\section{Automaton minimization}
\label{sec:dfa-min}

\myOmit{
Pesant's propagator for the \regular
constraint runs in $O(ndQ)$ time
where $n$ is the number of variables,
$d$ is the maximum domain size, and
$Q$ is the number of states of
the automaton \cite{pesant1}. If therefore we can
reduce the number of states in the automaton,
then we will speed up propagation. 
For instance, when an automaton is built
in Gecode, we can set a flag to minimize its 
size. }

The DFA constructed by this or other methods
may contain redundant states and transitions. 
We can speed up propagation of the \regular
constraint by minimizing the size of this automaton. 
Minimization can be either offline (i.e. before we have
the problem data and have unfolded the automaton) or online 
(i.e. once we have the problem
data and have unfolded the automaton). 
There are several reasons why we might 
prefer an online approach where we unfold before minimizing. 
First, although minimizing after unfolding
may be more expensive than minimizing
before unfolding, both are cheap to perform. Minimizing 
a DFA takes $O(Q \log Q)$ time 
using Hopcroft's algorithm and $O(nQ)$ time
for the unfolded DFA where $Q$ is the
number of states \cite{revuz1}. 
Second, thanks to Myhill-Nerode's theorem,
minimization does not change the
layered nature of the unfolded DFA.
Third, and perhaps most importantly, 
minimizing a DFA after unfolding
can give an exponentially smaller
automaton than minimizing the DFA and then
unfolding. To put it another way,
unfolding may destroy the minimality
of the DFA.

\begin{theorem}
Given any DFA ${\calA}$,
$|\min(\unfold_n({\calA}))| \ll
|\unfold_n(\min({\calA}))|$.
\end{theorem}
\myproof
To show
$|\min(\unfold_n({\calA}))| \leq
|\unfold_n(\min({\calA}))|$, 
we observe that both 
$\min(\unfold_n({\calA}))$ and
$\unfold_n(\min({\calA}))$ are automata
that recognize the same language. 
By definition, minimization returns
the smallest DFA accepting this language. 
Hence $\min(\unfold_n({\calA}))$ cannot be larger than
$\unfold_n(\min({\calA}))$.

\nina{
To show 
unfolding then minimizing can 
give an exponentially smaller sized DFA,
consider the following language $L$. A string of length $k$
belongs to $L$ iff it contains the symbol $j$, $j = k \ \mymod\ n$,
where $n$ is a given constant. 
The alphabet of the language $L$ is  $\{0,\ldots,n-1\}$.}
The minimal DFA
for this 
language has $\Omega(n2^n)$ states
as each state needs to record which symbols from
$0$ to $n-1$ have
been seen so far, as well as the current length of the
string mod $n$. Unfolding this minimal
DFA and restricting it to strings
of length $n$ gives an acyclic DFA
with $\Omega(n2^n)$ states.
\nina{Note that all strings are of length $n$
and the equation $j = n \ \mymod\ n$ has the single solution $j = 0$. Therefore, the language $L$
consists of the strings of length $n$ that contain the symbol $0$.}
On the other hand, if we unfold and then minimize,
we get an acyclic DFA with just $2n$ states.
Each layer of the DFA has two
states which record whether $0$ has been seen.
\qed

\myOmit{
\subsection{Minimization and non-determinism}

As we remarked earlier, 
Pesant's propagator for the \regular
constraint works equally well with NFAs as
with DFAs. 
However, NFAs offer some advantages over DFAs
for specifying \regular constraints \cite{qwcp06}. 
In particular, the smallest NFA for a regular 
language can be exponentially smaller 
than the smallest equivalent DFA. 

\begin{myexample}
Consider again the regular language $L$ used in the last proof. 
This contains all strings which mention
the symbol $j$ where $k$ is the length of the
strong and $j = k \ \mymod\ n$. 
The minimal DFA
for this regular language has $\Omega(n2^n)$ states.
However, there is an exponential smaller
NFA which accepts this same language. 
To see this, given any fixed $j$,
consider the regular language $L_j$
which contains all strings whose length mod $n$ equals
$j$ which also contain the symbol $j$. 
$L_j$ is accepted by the cyclic NFA ${\calA}_j$
with $O(n)$ states in which states record the
current length of the string mod $n$ and whether 
$j$ has been seen so far or not. 
To construct a NFA which accepts $L$, we
simply branch non-deterministically 
from the start state on any possible symbol
$0$ to $n-1$
into ${\calA}_j$ for $0 \leq j < n$. 
The resulting NFA has $O(n^2)$ states and 
accepts any string in $L$. 
\end{myexample}

Unfortunately, non-determinism introduces a significant cost
when we are minimizing. 
Minimization of a DFA is polynomial in the number
of states whilst minimization of a NFA is PSPACE-hard
\cite{msfocs72}. Even when we consider just the
acyclic NFA constructed by unfolding
a NFA, minimization remains 
NP-hard \cite{ajvwia99}. 
One option might be to minimize the NFA offline
since the offline cost may not be important
to us. 
Another option is to minimize online but
with heuristic methods which do not
guarantee the smallest possible NFA but
take only polynomial time. 
However, as with DFAs, there is a strong
reason to want to unfold a NFA before minimizing. 
Minimizing after unfolding
can give an exponentially smaller NFA. 

\begin{mytheorem}
Given any NFA ${\calA}$,
$|\min(\unfold_n({\calA}))| \ll
|\unfold_n(\min({\calA}))|$.
\end{mytheorem}
\myproof
By the definition of minimization,
 $|\min(\unfold_n({\calA}))| \leq
|\unfold_n(\min({\calA}))|$.
To show 
unfolding then minimizing can 
give an exponentially smaller sized automaton, 
consider the language $L'$ containing all strings
which mention the symbols $0$ to $j$ where $k$ is the length of
the string and $j = k \ \mymod\ n$. 
The minimal NFA
for this regular language has $\Omega(n2^n)$ states
as each state needs to record which symbols from $0$ to $n-1$ have
been seen so far, as well as the current length of the
string mod $n$. Unfolding this minimal
NFA and restricting it to strings
of length $n$ gives an acyclic NFA
with $\Omega(n2^n)$ states.
On the other hand, if we unfold and then minimize,
we get an acyclic NFA with just $2n$ states;
each layer of the NFA needs two
states to record if $0$ has been seen
so far or not. 
\qed

\subsection{Simplified automata}
}


Further,
if we make our initial problem
domain consistent, domains might be pruned
which give rise to possible simplifications
of the DFA. We show here that
we should also perform such simplification
before minimizing. 
\sloppy
\begin{theorem} \label{thm:simplified-dfa}
Given any DFA ${\calA}$,
$|\min(simplify(\unfold_n({\calA})))| \ll
|simplify(\min(\unfold_n({\calA})))|.$
\end{theorem}
\myproof
Both 
$\min(simplify(\unfold_n({\calA})))$ and
$simplify(\min(\unfold_n({\calA})))$ are DFAs
that recognize the same language of strings of length $n$. 
By definition, minimization must return
the smallest DFA accepting this language. 
Hence $\min(simplify(\unfold_n({\calA})))$ is
no larger than $simplify(\min(\unfold_n({\calA})))$.

To show that  minimization after simplification 
may give an exponentially smaller sized automaton,
\myOmit{
consider the language $L$ which contains
sequences of integers from $1$ to $n$ in which at least one 
integer is repeated (but not necessarily at
adjacent positions in the string). The minimal
unfolded DFA for strings
of length $n$ from this language has $\Omega(2^n)$ states
as each state needs to record which 
integers have been seen so far. Suppose the integer
$n$ is removed from the domain of each variable. 
We can remove the corresponding transitions 
in the minimal unfolded DFA. We can also remove
those states which record a substring containing $n$. 
However, the simplified DFA still has $\Omega(2^n)$ states
to record which integers 1 to $n-1$ have been
seen so far. On the other hand, suppose we simplify before we
minimize. By a simple pigeonhole argument, 
we have $n-1$ values taken by $n$ variables so one
value must necessarily be repeated. Hence, 
the simplified DFA accepts any string of 
length $n$. The minimal DFA accepting this
language has $n$ states. 
\qed

Note that such simplification can detect dis-entailment. 
If a \regular constraint is dis-entailed, then
there is no path from the start state to any of the final
states. Hence, simplification will delete the start state. 
What about detecting entailment? 
The example used in the last proof depends 
on a \regular constraint that is entailed. 
We can, however, adapt the example so that the
\regular constraint is not entailed but
minimization after simplification still
gives an exponential reduction.

\begin{myexample}
}
consider the language which contains
sequences of integers from $1$ to $n$ in which at least one 
integer is repeated and in which the last two integers 
are different.
\nina{The alphabet of the language $L$ is  $\{1,\ldots,n\}$.} 
The minimal
unfolded DFA for strings of
length $n$ from this language has $\Omega(2^n)$ states
as each state needs to record which 
integers have been seen.
Suppose the integer
$n$ is removed from the domain of each variable. 
The simplified DFA still has $\Omega(2^n)$ states
to record which integers 1 to $n-1$ have been
seen.
On the other hand, suppose we simplify before we
minimize. By a pigeonhole argument, we can ignore the
constraint that an integer is repeated. Hence we
just need to ensure that the string is of 
length $n$ and that the last two integers are
different. The minimal DFA accepting this
language requires just $O(n)$ states. 
\qed

\myOmit{
Similar results can be shown for NFAs.
That is, minimizing a NFA after simplification
is at least as good as 
minimizing the NFA before simplification, and
can sometimes be exponentially better. 

\begin{mytheorem}
Given any NFA ${\calA}$,
$|\min(\simplify(\unfold_n({\calA})))| \ll
|\simplify(\min(\unfold_n({\calA})))|.$
\end{mytheorem}
\myproof
Both 
$\min(\simplify(\unfold_n({\calA})))$ and
$\simplify(\min(\unfold_n({\calA})))$ are NFAs
that recognize the same language of strings of length $n$. 
By definition, minimization returns
the smallest NFA accepting this language. 
Hence $\min(simplify(\unfold_n({\calA})))$ is at least
as small as $simplify(\min(\unfold_n({\calA})))$.

To show that  minimization after simplification 
may give an exponentially smaller sized automaton
than minimization before simplification, 
consider the language $L$ which contains
sequences of integers from $1$ to $2n$ in which no
integer is repeated. The minimal
unfolded NFA for strings of length $n$ from this language
has $\Omega(2^{2n})$ states
as each state needs to record which 
of the $2n$ possible integers have been seen.
Suppose domains are pruned so that the
$i$th variable has just the integers $2i-1$ and $2i$ left
in its domain. The simplified NFA still has $\Omega(2^{2n})$ states. 
On the other hand, suppose we simplify before we
minimize. The automaton will accept every string
constructed from these reduced domains. 
The minimal NFA accepting this
language requires just $n$ states. 
\qed

The example used in the last proof is again 
a \regular constraint that is entailed. 
We can adapt the example so that the
\regular constraint is not entailed but
minimization after simplification still
gives an exponential reduction.

\begin{myexample}
Consider again the language $L$ which contains
sequences of integers from $1$ to $2n$ in which no
integer is repeated. The minimal
unfolded NFA for strings of length $n$ in this language
has $\Omega(2^{2n})$ states. 
Suppose domains are pruned so that the
the first two variables both
have the integers $1$ and $2$ in the domain,
and the $i$th variable for $i>2$ has just the integers $2i-1$ and $2i$ left
in its domain. Unlike the previous example, the \regular
constraint with these domains is not yet entailed. 
The simplified NFA still has $\Omega(2^{2n})$ states. 
On the other hand, suppose we simplify before we
minimize. The automaton will accept every string
of length $n$ in which the first 
two symbols are different. The minimal NFA accepting this
language requires just $O(n)$ states. 
\end{myexample}
}

\section{Empirical results}
\label{sec:experiments}

We empirically evaluated the results of our method on a set of
shift-scheduling benchmarks~\cite{Demassey05,CoteMIP07}~\footnote{ We
  would like to thank Louis-Martin Rousseau and Claude-Guy Quimper for
  providing us with the benchmark data}. Experiments were run with the
Minisat+ solver 
for pseudo-Boolean instances
and Gecode 2.2.0 for constraint problems,
on an Intel Xeon 4 CPU, 2.0 Ghz, 4G RAM.  We use
a timeout of $3600$ sec in all experiments.
The problem is to schedule employees to activities
subject to various rules, e.g. a
full-time employee has one hour for lunch.  This rules are
specified by a context-free grammar augmented with restrictions on
productions~\cite{qwcp07}.  A schedule for an
employee has $n = 96$ slots of 15 minutes represented by $n$ variables. In
each slot, an employee can work on an activity ($a_i$), take a break
($b$), lunch ($l$) or rest ($r$). These rules are specified by the
following grammar:
\begin{displaymath}
\begin{array}{ccc}
	S\rightarrow RPR,   f_P(i,j)\equiv 13 \leq j \leq 24,  \ &  P \rightarrow WbW, & L \rightarrow lL|l,   f_L(i,j)\equiv j = 4   \\
	S\rightarrow RFR,   f_F(i,j)\equiv 30 \leq j \leq 38,  \ & R \rightarrow rR|r, 	& W \rightarrow A_i,  f_W(i,j)\equiv j \geq 4 \\
	A_i \rightarrow a_iA_i|a_i, f_A(i,j)\equiv open(i), \ & F \rightarrow PLP  & \\
\end{array}
\end{displaymath}
where functions $f(i,j)$ are predicates that restrict the start and
length of any string matched by a specific production, and $open(i)$
is a function that returns $1$ if the business is open at $i^{th}$
slot and $0$ otherwise. In addition, the business requires a certain
number of employees working in each 
activity at given times during the day.
We
minimize the number of slots in which employees
work such that the demand is satisfied.  

As shown in~\cite{qwcp07}, this problem can be converted
into a pseudo-Boolean (PB) model.  The \CFG constraint is converted
into a SAT formula in conjunctive normal form using the \andor graph.
To model labour demand for a slot we introduce Boolean variables
$b(i,j,a_k)$, equal to $1$ if $j^{th}$ employee performs activity
$a_k$ at $i^{th}$ time slot. For each time slot $i$ and activity $a_k$
we post a pseudo-Boolean constraint $\sum_{j=1}^m b(i,j,a_k) >
d(i,a_k)$, where $m$ is the number of employees. The objective is
modelled using the function $\sum_{i=1}^n \sum_{j=1}^m \sum_{k=1}^a
b_{i,j,a_k}$. 
Additionally, the problem can be formulated as an optimization problem
in a constraint solver, using a matrix model with one row for each
employee. We post a \CFG constraint on each row, \among constraints on
each column for labour demand and \lex constraints between adjacent
rows to break symmetry. We use the static variable and value ordering
used in~\cite{qwcp07}.

We compare this with reformulating the \CFG constraint as a
\regular constraint.
Using algorithm~\ref{a:stack-graph}, we computed the
size of an equivalent NFA. Surprisingly, this is not too big, so we
converted the \CFG constraint to a DFA
then minimized. In order to reduce the blow-up that may occur
converting a NFA to a DFA, we heuristically 
minimized the NFA \nina{  using the following simple observation:
two states are equivalent if they have identical outgoing transitions. 
We traverse the NFA from 
the last to the first layer and merge equivalent states and then
apply the same procedure to the reversed NFA. We repeat
until we cannot find a pair of equivalent states. 
}
We also simplified the original \CYK
table, taking into account whether the business is open or closed at
each slot. Theorem~\ref{thm:simplified-dfa} suggests such
simplification can significantly reduce the size both of 
the \CYK table and of the resulting automata. \nina {In practice we also observe a significant 
reduction in size. The resulting minimized automaton obtained before simplification is about 
ten times larger compared to the minimised DFA obtained after simplification.} 
Table~\ref{t:t1} gives the sizes of representations at each step. 
We see from this that the minimized DFA is always smaller than the
original CYK table. 
Interestingly, the subset construction generates the minimum DFA from
the NFA, even in the case of two activities, and heuristic
\nina{minimization of the NFA achieves a notable reduction.}

\begin{table}
\begin{center}
 {\scriptsize
   \caption{\label{t:t1} Shift Scheduling Problems. $G_a$ is the
     acyclic grammar, $N_{a}^{\varepsilon}$ is NFA with
     $\varepsilon$-transitions, $N_a$ is NFA without
     $\varepsilon$-transitions, $\min(N_{a})$ is minimized NFA,
     $\calA$ is DFA obtained from $\min(N_{a})$, $\min({\calA})$ is
     minimized ${\calA}$, $a$ is the number of activities, $\#$ is the
     benchmark number.  }
\begin{tabular}{| c|c|cc|cc|cc|cc|cc|cc|cc|}
\hline
$\#act$ & 
$\#$ & 
\multicolumn {2}{|c|}{$G_a$} & 
\multicolumn {2}{|c|}{$NFA_{a}^{\varepsilon}$} & 
\multicolumn {2}{|c|}{$NFA_{a}$} & 
\multicolumn {2}{|c|}{$\min(NFA_{a})$} & 
\multicolumn {2}{|c|}{$DFA$} & 
\multicolumn {2}{|c|}{$\min(DFA)$} \\ 
\hline 
  & 
  & 
  terms & prods & 
  states & trans & 
  states & trans & 
  states & trans & 
  states & trans & 
  states & trans \\ 
\hline 
1 & 2/3/8 &
4678 &/ 9302&
69050 &/ 80975&
29003 &/ 42274&
\textbf{3556} &/ \textbf{4505}&
3683 &/ 4617&
3681 &/  4615 \\ 
1 & 4/7/10 &
3140 &/ 5541&
26737 &/ 30855&
11526 &/ 16078&
\textbf{1773} &/ \textbf{2296}&
1883 &/ 2399&
1881 &/  2397 \\ 
1 & 5/6 &
2598 &/ 4209&
13742 &/ 15753&
5975 &/ 8104&
\textbf{1129} &/ \textbf{1470}&
1215 &/ 1553&
1213 &/  1551 \\ 
\hline 
2 & 1/2/4 &
3777 &/ 6550&
42993 &/ 52137&
19654 &/ 29722&
\textbf{3157} &/ \textbf{4532}&
3306 &/ 4683&
3303 &/  4679 \\ 
2 & 3/5/6 &
\textbf{5407} &/ {10547}&
111302 &/ 137441&
50129 &/ 79112&
{5975} &/ \textbf{8499}&
6321 &/ 8846&
6318 &/  8842 \\ 
2 & 8/10 &
\textbf{6087} &/ {12425}&
145698 &/ 180513&
65445 &/ 104064&
{7659} &/ \textbf{10865}&
8127 &/ 11334&
8124 &/  11330 \\ 
2 & 9 &
4473 &/ 8405&
76234 &/ 93697&
34477 &/ 53824&
\textbf{4451} &/ \textbf{6373}&
4691 &/ 6614&
4688 &/  6610 \\ 
\hline 
\end{tabular} 

 } 
\end{center} 
\end{table} 

For each instance, we used the resulting DFA in place of the
\CFG constraint in both the CP model and the PB model using
the encoding of the \regular constraint (DFA or NFA) into CNF 
\cite{Bacchus07GAC}. 
We compare the model that uses the PB encoding of the \CFG constraint
($\CFGSM_1$) with two models that use the PB encoding of the \regular
constraint ($\regular_1$, $\regular_2$), a CP model that uses the \CFG constraint 
($\CFGSM_1^{CP}$) and a CP model that uses a
\regular constraint ($\regular_1^{CP}$).
$\regular_1$ and  $\regular_1^{CP}$ use the DFA, whilst $\regular_2$ 
uses the NFA constructed 
after simplification by when the business is closed.

\nina{The performance of a SAT solver can be sensitive to the ordering of the clauses in the formula. To
test robustness of the models, we randomly shuffled each of PB instances to generate 10 equivalent problems 
and averaged the results over 11 instances. Also, 
the \CFG and \regular constraints were encoded into a PB formula in two different ways.
The first encoding ensures that unit propagation enforces domain
consistency on the constraint. The 
second encoding ensures that UP detects disentailment
of the constraint, but does not always enforce domain consistency.
For the \CFG constraint we omit the same set of clauses as in \cite{qwcp07} to obtain 
the weaker PB encoding. 
For the \regular constraint we omit the set of clauses that performs
the backward propagation of the \regular constraint. 
Note that Table~\ref{t:t2} shows the median time and the number of backtracks 
to prove optimality over 11 instances. 
For each model we show the best median time 
and the corresponding number of backtracks for
the PB encoding that achieves domain consistency and for the weaker encoding.}

\begin{table}
\begin{center}
 {\scriptsize
   \caption{\label{t:t2} Shift Scheduling Problems. $\CFGSM_1$ is the PB
     model with $\CFG$, $\regular_1$ is the PB model with
     $\min(\simplify(\mathit{DFA}))$, $\regular_2$ is the PB model with
     $\min(\simplify(\mathit{NFA}))$, $\CFGSM_1^{CP}$ is the CSP model with
     $\CFG$, $\regular_1^{CP}$ is the CSP model with
     $\min(\simplify(\mathit{DFA}))$. We show \textbf{t}ime 
     and number of \textbf{b}acktracks to prove optimality (the median time and the median number of backtracks for the PB encoding
     over \textbf{s}olved  shuffled instances),
     number of \textbf{a}ctivities, the
     number of \textbf{w}orkers and the benchmark number \#. }
\begin{tabular}{|c|c|c||c|c|rl|c|c|rl|c|c|rl||c|rl|c|rl|}
\hline
&&& \multicolumn{12}{|c||}{PB/Minisat+}
& \multicolumn{6}{|c|}{CSP/Gecode}
\\
\hline
$a$ & 
$\#$ & 
$w$ & 
\multicolumn {4}{|c|}{$\CFGSM_1$} & 
\multicolumn {4}{|c|}{$\regular_1$} & 
\multicolumn {4}{|c||}{$\regular_2$} & 
\multicolumn {3}{|c|}{$\CFGSM_1^{CP}$} & 
\multicolumn {3}{|c|}{$\regular_1^{CP}$} \\ 
\hline 
  & 
  & 
  & 
 cost & s& $t$ &/ $b$ & 
  cost & s& $t$ &/ $b$  &
  cost & s& $t$ &/ $b$  &
  cost &  $t$ &/ $b$  &
  cost &  $t$ &/ $b$  \\
\hline 
\hline 
1 & 2& 4 &
\textbf{26.00} & 11 &      27 &/ {    8070}  
 & \textbf{26.00} & 11 &       9 &/ {   11053}  
 & \textbf{26.00} & 11 &\textbf{       4} &/ \textbf{    7433}  
 & {26.75} &- &/ -  
 & \textbf{26.00} &- &/ -  
 \\ 
1 & 3& 6 &
\textbf{36.75} & 11 &     530 &/ {  101560}  
 & \textbf{36.75} & 11 &      94 &/ {   71405}  
 & \textbf{36.75} & 11 &\textbf{      39} &/ \textbf{   58914}  
 & {37.00} &- &/ -  
 & {37.00} &- &/ -  
 \\ 
1 & 4& 6 &
\textbf{38.00} & 11 &      31 &/ {   16251}  
 & \textbf{38.00} & 11 &      12 &/ {   10265}  
 & \textbf{38.00} & 11 &\textbf{       6} &/ \textbf{    7842}  
 & \textbf{38.00} &- &/ -  
 & \textbf{38.00} &- &/ -  
 \\ 
1 & 5& 5 &
\textbf{24.00} & 11 &       5 &/ {    3871}  
 & \textbf{24.00} & 11 &       2 &/ {    4052}  
 & \textbf{24.00} & 11 &\textbf{       2} &/ \textbf{    2598}  
 & \textbf{24.00} &- &/ -  
 & \textbf{24.00} &- &/ -  
 \\ 
1 & 6& 6 &
\textbf{33.00} & 11 &       9 &/ {    5044}  
 & \textbf{33.00} & 11 &       4 &/ {    4817}  
 & \textbf{33.00} & 11 &\textbf{       3} &/ \textbf{    4045}  
 & - &- &/ -  
 & \textbf{33.00} &- &/ -  
 \\ 
1 & 7& 8 &
\textbf{49.00} & 11 &      22 &/ {    7536}  
 & \textbf{49.00} & 11 &       9 &/ {    7450}  
 & \textbf{49.00} & 11 &\textbf{       7} &/ {    8000}  
 & \textbf{49.00} &- &/ -  
 & \textbf{49.00} &- &/ -  
 \\ 
1 & 8& 3 &
\textbf{20.50} & 11 &      13 &/ {    4075}  
 & \textbf{20.50} & 11 &       4 &/ {    5532}  
 & \textbf{20.50} & 11 &\textbf{       2} &/ \textbf{    1901}  
 & {21.00} &- &/ -  
 & \textbf{20.50} &      92 &/ {2205751}  
 \\ 
1 & 10& 9 &
\textbf{54.00} & 11 &     242 &/ {  106167}  
 & \textbf{54.00} & 11 &     111 &/ \textbf{   91804}  
 & \textbf{54.00} & 11 &\textbf{     110} &/ {  109123}  
 & - &- &/ -  
 & - &- &/ -  
 \\ 
\hline 
2 & 1& 5 &
\textbf{25.00} & 11 &      92 &/ {   35120}  
 & \textbf{25.00} & 11 &      96 &/ {   55354}  
 & \textbf{25.00} & 11 &\textbf{      32} &/ \textbf{   28520}  
 & \textbf{25.00} &- &/ -  
 & \textbf{25.00} &      90 &/ {1289554}  
 \\ 
2 & 2& 10 &
\textbf{58.00} & 1 &    3161 &/ \textbf{  555249}  
 & \textbf{58.00} & 0 &- &/ -  
 & \textbf{58.00} & 4 &\textbf{    2249} &/ {  701490}  
 & - &- &/ -  
 & \textbf{58.00} &- &/ -  
 \\ 
2 & 3& 6 &
\textbf{37.75} & 0 &- &/ -  
 & \textbf{37.75} & 1 &    3489 &/ {  590649}  
 & \textbf{37.75} & 9 &\textbf{    2342} &/ \textbf{  570863}  
 & {42.00} &- &/ -  
 & {40.00} &- &/ -  
 \\ 
2 & 4& 11 &
\textbf{70.75} & 0 &- &/ -  
 & {71.25} & 0 &- &/ -  
 & {71.25} & 0 &- &/ -  
 & - &- &/ -  
 & - &- &/ -  
 \\ 
2 & 5& 4 &
\textbf{22.75} & 11 &     739 &/ {  113159}  
 & \textbf{22.75} & 11 &     823 &/ {  146068}  
 & \textbf{22.75} & 11 &\textbf{     308} &/ \textbf{   69168}  
 & {23.00} &- &/ -  
 & {23.00} &- &/ -  
 \\ 
2 & 6& 5 &
\textbf{26.75} & 11 &      86 &/ {   25249}  
 & \textbf{26.75} & 11 &     153 &/ {   52952}  
 & \textbf{26.75} & 11 &\textbf{      28} &/ \textbf{   21463}  
 & \textbf{26.75} &- &/ -  
 & \textbf{26.75} &- &/ -  
 \\ 
2 & 8& 5 &
\textbf{31.25} & 11 &    1167 &/ {  135983}  
 & \textbf{31.25} & 11 &     383 &/ {  123612}  
 & \textbf{31.25} & 11 &\textbf{      74} &/ \textbf{   47627}  
 & {32.00} &- &/ -  
 & {31.50} &- &/ -  
 \\ 
2 & 9& 3 &
\textbf{19.00} & 11 &    1873 &/ {  333299}  
 & \textbf{19.00} & 11 &     629 &/ {  166908}  
 & \textbf{19.00} & 11 &\textbf{     160} &/ \textbf{  131069}  
 & {19.25} &- &/ -  
 & \textbf{19.00} &- &/ -  
 \\ 
2 & 10& 8 &
\textbf{55.00} & 0 &- &/ -  
 & \textbf{55.00} & 0 &- &/ -  
 & \textbf{55.00} & 0 &- &/ -  
 & - &- &/ -  
 & - &- &/ -  
 \\ 
\hline 
\end{tabular} 
} 
\end{center} 
 \end{table} 

Table~\ref{t:t2} shows the results of our experiments using these
5 models.
%
\nina{The model $\regular_2$  outperforms $\CFGSM_1$ in all
benchmarks, whilst model $\regular_1$ outperforms $\CFGSM_1$ in most of the benchmarks. 
The model $\regular_2$ also proves optimality in several instances of hard benchmarks. 
It should be noted that performing simplification before minimization is essential. It significantly reduces the size of the encoding and speeds up MiniSat+ by
factor of 5\footnote{Due to lack of space we do not show these results}.
}
Finally, we note that the PB models consistently outperformed 
the CP models, in agreement with the observations
of~\cite{qwcp07}. Between the two CP models, 
$\regular_1^{CP}$ is significantly better than
$\CFGSM_1^{CP}$, finding a better solution in many instances and proving
optimality in two instances. In addition, although we do not show it
in the table, Gecode is approximately three orders of magnitude faster
per branch with the $\regular_1^{CP}$ model. For instance, in
benchmark number 2 with 1 activity and 4 workers, it explores
approximately 80 million branches with the $\regular_1^{CP}$ and 24000
branches with the $\CFGSM_1^{CP}$ model within the 1 hour timeout.

\section{Other related work}

%

Beldiceanu \emph{et al}~\cite{BCP04} and Pesant~\cite{pesant1}
proposed specifying constraints using automata and provided filtering
algorithms for such specifications. 
Quimper and Walsh \cite{qwcp06} and Sellmann \cite{grammar2}
then independently proposed the \CFG constraint. Both
gave a monolithic propagator based on
the CYK parser.  Quimper and Walsh~\cite{qwcp07} proposed
a CNF decomposition of the \CFG constraint, while
Bacchus~\cite{Bacchus07GAC} proposed a CNF decomposition of the
\regular constraint.
Kadioglu and Sellmann~\cite{KS08}
improved the space efficiency of 
the propagator for the \CFG constraint by a factor of $n$. \nina{
Their propagator was evaluated on the same shift scheduling 
benchmarks as here. However, as they only found feasible solutions and
did not prove optimality, their 
results are not directly comparable. }
C\^{o}t\'{e}, Gendron, Quimper and Rousseau
proposed a mixed-integer programming (MIP) encoding
of the \CFG constraint \cite{CoteMIP07},
Experiments on the same shift scheduling problem
used here
show that such encodings are 
competitive.

There is a body of work on other methods
to reduce the size of constraint representations. 
Closest to this work is
Lagerkvist who observed that
a \regular constraint represented as a multi-value decision
diagram (MDD) is no larger than 
that represented by a DFA that is minimized
and then unfolded \cite{lagerkvistthesis}. 
A MDD is similar to an unfolded and then minimized
DFA except a MDD can have long edges which skip over
layers. We extend this observation by proving
an exponential separation
in size between such representations. 
As a second example, Katsirelos and Walsh 
compressed table constraints representing allowed or disallowed tuples
using decision tree methods 
\cite{kwcp07}. They also used a compressed
representation for tuples that can provide exponentially
savings in space. 
As a third example, Carlsson proposed
the \mycase constraint which can be represented
by a DAG where each node represents a range of
values for a variable, and a path from the root to
a leaf represents a set of satisfying assignments
\cite{carlssoncase}.

\section{Conclusions}

We have shown how to transform a \CFG constraint
into a \regular constraint specified. 
In the worst case, the transformation may increase
the space required to represent the constraint. 
However, in practice, we observed that such
transformation reduces the space required to represent 
the constraint and speeds up propagation. 
We argued that transformation also permits us
to compress the representation using standard
techniques for automaton minimization. 
We proved that minimizing such automata after
they have been unfolded and domains initially
reduced can give automata that are
exponentially more compact than those obtained by
minimizing before unfolding and reducing. 
Experimental results demonstrated 
that such transformations can improve the size of rostering
problems that can be solved. 

\myOmit{
There are many directions for future work.
For example, we intend to look at (heuristic)
techniques to minimize the
\andor graph of the \CYK algorithm.
As a second example, we are exploring
other methods for minimizing the NFA.
For instance, perhaps we can use a SAT encoding
to perform this minimization?
As a third example, we are considering
incrementally minimizing the DFA during search
itself. As domains are pruned, we will
be able to reduce further the size of the DFA.}

\bibliographystyle{splncs}

\bibliography{/home/gkatsi/research/gkatsi}


\end{document}